\documentclass[journal]{IEEEtran}

\usepackage{geometry}

\usepackage[utf8]{inputenc} 
\usepackage[T1]{fontenc}    
\usepackage{hyperref}       
\usepackage{url}            
\usepackage{booktabs}       
\usepackage{amsfonts}       
\usepackage{nicefrac}       
\usepackage{microtype}      
\usepackage{lipsum}
\usepackage{fancyhdr} 
\usepackage{graphicx}       
\usepackage[cmex10]{amsmath} 
\usepackage[T1]{fontenc}
\usepackage{ifthen}
\usepackage{cite}
\usepackage{color}
\usepackage[dvipsnames]{xcolor}
\usepackage{pifont} 
\usepackage{latexsym}
\usepackage{amssymb,mathtools,amsthm}
\usepackage{amsfonts}
\usepackage{dsfont}
\usepackage{amsopn}
\usepackage{array}
\def\epsilon{\varepsilon}
\def\eps{\varepsilon}
\usepackage{algorithm}
\usepackage{algpseudocode}
\usepackage{graphicx}

\newcommand{\greencheck}{{\color{black}\ding{51}}} 
\newcommand{\redx}{{\color{black}\ding{55}}} 

\newtheorem{theorem}{Theorem}

\newtheorem{proposition}[theorem]{Proposition}
\newtheorem{definition}{Definition}

\allowdisplaybreaks

\begin{document}
\onecolumn
\title{Differentially Private Fair Binary Classifications
%
}
\author{
  Hrad Ghoukasian, Shahab Asoodeh \\
  Department of Computing and Software, McMaster University \\
  \texttt{ \{ghoukash, asoodeh\}@mcmaster.ca} \\
}

\maketitle

\begin{abstract}
In this work, we investigate binary classification under the constraints of both differential privacy and fairness.
    We first propose an algorithm based on the decoupling technique for learning a classifier with only fairness guarantee. This algorithm takes in classifiers trained on different demographic groups and generates a single classifier satisfying statistical parity. We then refine this algorithm to incorporate differential privacy. The performance of the final algorithm is rigorously examined in terms of privacy, fairness, and utility guarantees.
   Empirical evaluations conducted on the Adult and Credit Card datasets illustrate that our algorithm outperforms the state-of-the-art in terms of fairness guarantees, while maintaining the same level of privacy and utility.      
\end{abstract}


\section{Introduction}
Machine learning algorithms are increasingly utilized in high-stake decision-making processes, highlighting the need to thoroughly assess their trustworthiness. Two main topics in the context of trustworthy machine learning are privacy and fairness.  Machine learning models are required to, on the one hand, protect the privacy of individuals in the training datasets, and on the other hand, avoid leading to discrimination against any demographic subgroups. Differential privacy (DP) \cite{dwork2006calibrating,dwork2006differential} is the de-facto standard for privacy-preserving machine learning algorithms deployed in practice (e.g.,~\cite{erlingsson2014rappor,Apple_Privacy,Facebook020GuidelinesFI,rogers2020linkedin}). Informally speaking, a randomized algorithm is differentially private if its output distribution does not significantly change by changing an entry in the dataset (yielding a pair of neighboring datasets). 
However, unlike privacy, there is no universal definition of fairness. Therefore, how fairness is defined or algorithmically enforced depends on the particular context of the problem. A widely recognized concept of fairness is statistical parity \cite{feldman2015certifying}, also known as demographic parity. This definition implies that predictions of a classification model should be independent of the \textit{sensitive} attributes, e.g., gender or race. While privacy and fairness have been extensively studied separately in the literature, the intersection of them has recently gained attention.

In the literature on fair classification algorithms, incorporating sensitive attributes is known as ``fairness through awareness'' \cite{dwork2012fairness}. 
While the choice to use or not use sensitive attributes is specific to each problem, \cite{dwork2018decoupled} demonstrated that in scenarios where it is legally and ethically tenable to use these attributes, training separate classifiers (decoupled classifiers) for each group can outperform the performance of a single optimal classifier in terms of both accuracy and fairness (without privacy included). \cite{wang2021split} also showed that employing decoupled classifiers does not negatively impact any group in terms of average performance metrics in the information-theoretic regime where the underlying data distribution is known. 
%
%
%
However, a significant challenge emerges when considering privacy. It has been empirically demonstrated that differentially private mechanisms can exacerbate unfairness \cite{bagdasaryan2019differential,pujol2020fair}. This necessitates the implementation of specific strategies to mitigate the negative effects of DP on fairness guarantees. 
Combining the above idea of decoupled classification with differential privacy, \cite{stratification} proposed a ``stratification'' technique to reduce the disparity in differentially private mechanisms. This technique involves applying a DP mechanism separately to different subgroups and then recombining the respective results to derive overall statistics for the entire dataset.
It was demonstrated that a naive stratification approach can produce highly accurate estimates for population-level statistics without requiring an extra privacy budget. Building on this foundation, our pipeline is structured to first apply DP and subsequently address fairness using a post-processing step. This order of application is crucial because DP methods introduce random noise into the process, which can alter the predictions of our model. If the model met fairness standards before we applied DP, the introduction of this noise could lead to a situation where those fairness standards are no longer met.

One line of work explores the relationship between DP and different notions of fairness, examining their compatibility and determining how fairness deteriorates DP guarantees and vice versa\cite{bagdasaryan2019differential,cummings2019compatibility,mangold2023differential,changshokri2021privacy,pujol2020fair,farrand2020neither,agarwal_2020}. Another research direction at the intersection of privacy and fairness centers around developing classification algorithms that simultaneously guarantee DP and a specific notion of fairness. Despite these remarkable recent advances, most such works suffer from inherent limitations. For instance, some of these methods are tailored only to certain types of classification models, such as logistic regression, and are not universally applicable to all desired models\cite{xu2019achieving,jagielski2019differentially,ding2020differentially}. Other studies focus exclusively on providing privacy with respect to sensitive attributes of individuals in the dataset, neglecting the privacy of other features. More specifically, some methods ensure that an individual's sensitive attributes remain confidential, but they do not provide the same guarantees for non-sensitive attributes, leaving them potentially vulnerable to leakage \cite{jagielski2019differentially,mozannar2020fair,tran2021b,tran2022sf}. For instance,  \cite{mozannar2020fair} introduced two effective differentially private fair classification algorithms which provide privacy guarantees only with respect to the sensitive attributes. This was achieved through the use of randomized response for sensitive attributes. They used randomized response mechanism (for achieving the \textit{local} version of DP), and thus their method does not seem to naturally adapt to 
provide privacy guarantees for all features, especially when dealing with continuous and potentially high-dimensional non-sensitive attributes. Lastly, although most of these methods have shown practical effectiveness, they often lack theoretical guarantees regarding utility and the extent of fairness violation \cite{xu2019achieving,ding2020differentially,xu2021removing,tran2021a,tran2021b,tran2022sf,esipova2023disparate,lowy2023stochastic,yaghini2023learning}. A summary of the characteristics of these methods can be found in Table~\ref{table:features}.\footnote{
In this work, we allow access to the sensitive attributes at test time, similar to the approach in \cite{jagielski2019differentially}, \cite{mozannar2020fair}, and \cite{esipova2023disparate}.}

Our objective is to develop a binary classification algorithm with provable DP and fairness guarantees that addresses the limitations highlighted in Table \ref{table:features}. Inspired by the success of decoupled classifiers and the effectiveness of stratification in reducing disparate impact in differentially private mechanisms (as noted in \cite{wang2021split,dwork2018decoupled, stratification}), our approach begins with separate classifiers for each sub-population. These classifiers then go into a post-processing step to generate a single classifier.
Following the method used in \cite{jiang2020wasserstein}, our goal is to apply a post-processing technique that achieves statistical parity, ensuring that only minimal changes are made to the predictions of the original classifiers. We begin by slightly modifying the approach described in \cite{inherent}, initially developed for non-private settings (Algorithm \ref{algorithm1}). Then, we introduce Algorithm \ref{algorithm2}, a new method for binary classification that is both differentially private and fair. 
This new algorithm comes with theoretical guarantees for utility, fairness,  and differential privacy. More precisely, our contributions are as follows:
\begin{itemize}
    \item We establish a lower bound for the sum of prediction changes across a pair of subgroups under statistical parity (without privacy) and propose an algorithm (based on \cite[Algorithm 1]{inherent}) that attains it (Algorithm~\ref{algorithm1}). Theoretical guarantee of this algorithm is given in  Theorem \ref{algorithm1guarantees}.
    \item
    We then propose a differentially private version (Algorithm \ref{algorithm2}) of Algorithm~\ref{algorithm1} and derive its theoretical guarantees for fairness and utility ---both with high probability and in expectation--- in Theorem~\ref{alg2WHP} and Proposition \ref{alg2Exp}.
    \item 
    Through several experiments on two well-known datasets (namely, Adult and Credit Card), we empirically demonstrate that Algorithm \ref{algorithm2} achieves competitive accuracy when compared to the state-of-the-art DP-Fair classification method, DP-FERMI \cite{lowy2023stochastic}. In particular, we show that for a given level of accuracy and privacy, our algorithm provides a significantly better fairness guarantee across both datasets.\footnote{The experimental code can be accessed at \url{https://github.com/hradghoukasian/dp_fair_binary}.} 
\end{itemize}
\begin{table}[t]
\caption{Benchmark methods in DP-Fair classification}
\label{table:features}
\centering
\renewcommand{\arraystretch}{1.4} 
\begin{tabular}{|
  >{\centering\arraybackslash}m{0.14\columnwidth}|
  >{\centering\arraybackslash}m{0.15\columnwidth}|
  >{\centering\arraybackslash}m{0.15\columnwidth}|
  >{\centering\arraybackslash}m{0.15\columnwidth}|
  >{\centering\arraybackslash}m{0.15\columnwidth}|}
\hline
\textbf{Reference} & \textbf{Applicable to any model} & \textbf{Theoretical guarantee} & \textbf{Privacy w.r.t.\ all features} \\
\hline
\cite{xu2019achieving}
& \redx & \redx & \greencheck \\
\hline
\cite{jagielski2019differentially} (post-proc.)& \greencheck & \greencheck & \redx \\
\hline
\cite{jagielski2019differentially} (in-proc.)& \redx & \greencheck & \greencheck \\
\hline
\cite{ding2020differentially}  & \redx & \redx & \greencheck \\
\hline
\cite{mozannar2020fair}& \greencheck & \greencheck & \redx \\
\hline
\cite{xu2021removing} & \greencheck & \redx & \greencheck \\
\hline
\cite{tran2021a} & \greencheck & \redx & \greencheck \\
\hline
\cite{tran2021b} &  \greencheck &  \redx & \redx \\
\hline
\cite{tran2022sf} & \greencheck & \redx & \redx \\
\hline
\cite{esipova2023disparate} & \greencheck & \redx & \greencheck \\
\hline
\cite{lowy2023stochastic} & \greencheck &  \redx & \greencheck  \\
\hline
\cite{yaghini2023learning} (FairDPSGD)& \greencheck & \redx & \greencheck\\ 
\hline
\cite{yaghini2023learning} (FairPATE)& \greencheck & \redx & \greencheck\\ 
\hline
This Work &  \greencheck & \greencheck & \greencheck \\ 
\hline
\end{tabular}
\end{table}

\section{Notation and Basic Definitions}\label{notation_definition}
We consider a binary classification setting where there is a joint distribution $\mu$ over the triplet $T = (X,A,Y)$, where $X \in \mathcal{X} \subset \mathbb{R}^d$ is the feature vector of non-sensitive attributes, $ A \in \{0, 1\}$ is the sensitive attribute, and $Y \in \{0, 1\}$ is the target output. We use $\mu(Y)$ to denote the marginal distribution of $Y$ from a joint distribution $\mu$ over $Y$ and some other random variables. Denote the marginal distribution of input $X$ by $\mu^X$. For $a \in \{0, 1\}$, we use $\mu_a$ to mean the conditional distribution of $(X,Y)$ conditioned on $A = a$, and $\mu^X_a$ to mean the marginal distribution of input 
$X$ given $A=a$. For any group-aware classifier $h: \mathcal{X}\times\{0,1\}\rightarrow \{0,1\} $ and $a\in\{0,1\}$, we also use $h_a(\cdot)$ to denote the restriction of $h$ on $A = a$, respectively, i.e., $h_a(\cdot) := h(\cdot,a)$. Finally, the probabilistic inequality $X \leq_{\eta} Y$ for a pair of random variables $(X, Y)$  denotes the mathematical statement that $\mathbb P(X>Y)\leq \eta$.

We now formally define differential privacy and statistical parity. 
\begin{definition}
    (Differential privacy \cite{dwork2006calibrating,dwork2006differential}). A randomized mechanism $M:\mathcal{D}\rightarrow\mathcal{R}$ with domain $\mathcal{D}$ and range $\mathcal{R}$ is $(\epsilon,\delta)$-differentially private ($(\epsilon,\delta)$-DP) if for any pair of neighboring datasets $D$ and $D^\prime$ that differ in exactly one record, and for any subsets of outputs $S\subseteq\mathcal{R}$, we have,
    \begin{equation*}
        \mathbb{P}(M(D) \in S) \leq e^\epsilon\mathbb{P}(M(D^\prime) \in S) + \delta.
    \end{equation*}
\end{definition}
\begin{definition} \label{statistical_parity_gap}
(Statistical parity \cite{feldman2015certifying}). Given a joint distribution $\mu$, the statistical parity gap of a binary classifier $\hat{Y}=\hat{h}(X,A)$ with $\hat{h}: \mathcal{X}\times\{0,1\}\rightarrow \{0,1\}$ is 
\begin{equation*}
    \Delta_{SP}(\hat{h}) := |\mu_0(\hat{Y} = 1) - \mu_1(\hat{Y} = 1)|.
\end{equation*}
We say that $\hat Y = \hat{h}(X,A)$ satisfies $\gamma$-statistical parity if
    \begin{equation*}
        \Delta_{SP}(\hat{h}) \leq \gamma.
    \end{equation*}
\end{definition}

Following \cite{inherent, xu2019achieving, yaghini2023learning, ding2020differentially}, we adopt statistical parity as a metric for fairness.

\section{Main Results}\label{technical_results}


In this section, we develop a framework for designing a binary classification algorithm with provable DP and fairness guarantees. We begin in Section \ref{fair_post-proc_no_privacy} by detailing our approach to achieve fairness and the specified utility target in a non-private setting. Here, as outlined in Introduction, our framework involves partitioning the dataset according to sensitive attributes. We first develop a separate classifier for each subgroup. We then implement a post-processing technique that carefully combines those decoupled classifiers in a way to achieve statistical parity, with the objective of minimally perturbing the original classifiers' predictions. Then in Section \ref{fair_post-proc_with_privacy}, we expand this framework to include DP.  

We consider two classifiers, $h_0^*: \mathcal{X}\rightarrow \{0,1\}$ and $h_1^*: \mathcal{X}\rightarrow \{0,1\}$, each trained on subgroups specified by the sensitive attribute $A$ with values 0 and 1, respectively. These classifiers are designed to maximize accuracy without initially considering fairness constraints. In a non-private setting, $h_0^*$ and $h_1^*$ are standard classifiers, whereas in private settings, they are considered classifiers learned by DP guarantees. Our post-processing method
aims to derive a fair classifier $\hat{h}: \mathcal{X} \times \{0,1\} \rightarrow \{0,1\}$ from $h_0^*$ and $h_1^*$. Following the methodology of \cite{jiang2020wasserstein}, our objective is to achieve this goal by minimally perturbing the predictions of the original classifiers. We thus seek fair $\hat h$ that optimizes the utility measured in terms of the sum $\mathbb{P}_{\mu^X_0}(\hat{h}_0(X) \neq h^*_0(X))+\mathbb{P}_{\mu^X_1}(\hat{h}_1(X) \neq h^*_1(X))$.
The reason for adopting the prediction changes over $\mu^{X}_{0}$ and $\mu^{X}_{1}$ (as opposed to the combined distribution $\mu^X$) is as follows: 
when the demographic subgroups are imbalanced in the overall population, relying only on prediction changes across the combined distribution $\mu^X$ can be misleading. This approach may hide significant prediction shifts of the less-represented group, which would be more apparent if we examined the sum of the prediction changes within each subgroup's distribution ($\mu_0^X$ and $\mu_1^X$). In other words, it might be possible to have small overall prediction changes across the combined distribution $\mu^X$ while minority groups are experiencing substantial changes in their predictions \cite{inherent}.  

In Section \ref{fair_post-proc_no_privacy}, we explore the non-private scenario, discussing Algorithm \ref{algorithm1} that ensures statistical parity and optimal utility.  In Section \ref{fair_post-proc_with_privacy}, we extend the algorithm to take privacy into consideration. In particular, we propose Algorithm \ref{algorithm2}, which outputs a classifier that guarantees DP and achieves statistical parity while maintaining minimal changes in the predictions of the original classifiers, both in expectation and with high probability. 
To discuss our utility metric under the constraint of statistical parity, we rely on the following proposition.

\begin{proposition}\label{proposition1}
Let $h_0^*: \mathcal{X}\rightarrow \{0,1\}$ and $h_1^*: \mathcal{X}\rightarrow \{0,1\}$, be arbitrary classifiers trained on subgroups specified by the sensitive attribute $A=0$ and $A=1$, respectively. If a predictor $\hat{Y}=\hat{h}(X,A)$ satisfies $\gamma$-statistical parity, then
\begin{equation*}
    \begin{split}
        \mathbb{P}_{\mu^X_0}(\hat{h}_0(X) \neq h^*_0(X))+\mathbb{P}_{\mu^X_1}(\hat{h}_1(X) \neq h^*_1(X)) \geq  \left| \mathbb{P}_{\mu^X_0}(h^*_0(X) = 1 ) - \mathbb{P}_{\mu^{X}_{1}}(h^*_1(X) = 1 )\right| - \gamma.
    \end{split}
\end{equation*}
\end{proposition}

\subsection{Fair Post-Processing Without Privacy}\label{fair_post-proc_no_privacy}
Let  $h_0^*: \mathcal{X}\rightarrow \{0,1\}$ and $h_1^*: \mathcal{X}\rightarrow \{0,1\}$ be decoupled classifiers, that is they are trained on subgroups associated with the sensitive attribute $A=0$ and $A=1$, respectively. $h_0^*$ and $h_1^*$ can be any arbitrary classifiers. In the context of our analysis, they can be considered as the best available classifiers trained on $\mu_0$ and $\mu_1$. 
Our goal is to develop an optimal fair classifier $\hat{h}: \mathcal{X}\times\{0,1\}\rightarrow \{0,1\}$ using $h_0^*$ and $h_1^*$. More precisely, we seek $\hat h$ that satisfies $\Delta_{SP}(\hat{h}) = 0$ while attaining the lower bound in Proposition~\ref{proposition1} for $\gamma = 0$:
\begin{equation*}
\begin{split}
    \mathbb{P}_{\mu^X_0}(\hat{h}_0(X) \neq h^*_0(X))+\mathbb{P}_{\mu^X_1}(\hat{h}_1(X) \neq h^*_1(X)) = \left| \mathbb{P}_{\mu^X_0}(h^*_0(X) = 1 ) - \mathbb{P}_{\mu^{X}_{1}}(h^*_1(X) = 1 )\right|.
\end{split}
\end{equation*}
We present Algorithm \ref{algorithm1} to address this objective. This algorithm, which is a slightly modified version of the algorithm in \cite{inherent}, constructs the fair classifier $h^*_{\text{Fair}}$. The original algorithm in \cite{inherent} builds a fair optimal classifier assuming oracle access to the Bayes optimal classifiers $h_0^*$ and $h_1^*$. Our ultimate goal, to be discussed in the next section with Algorithm~\ref{algorithm2}, is to identify a classifier that is both private and fair. However, the assumption of having access to Bayes optimal classifiers is not practical in DP settings, primarily due to the necessity of introducing noise. 
 Consequently, in Algorithm~\ref{algorithm1}, $h_0^*$ and $h_1^*$ are considered to be any arbitrary classifiers trained on subgroups specified by the sensitive attribute $A=0$ and $A=1$ (not necessarily the Bayes optimal classifiers). 


\begin{algorithm}
\caption{Optimal Fair Binary Classifier}
\label{algorithm1}
\begin{algorithmic}[1] 

\Statex \textbf{Input:} Classifiers $h_0^*$ and $h_1^*$
\Statex \textbf{Output:} A randomized classifier $h^*_{\text{Fair}}:\mathcal{X}\times\{0,1\}\rightarrow\{0,1\}$

\State Compute $\alpha = \mathbb{P}_{\mu^X_0} (h^*_0(X) = 1 )$ and $\beta = \mathbb{P}_{\mu^X_1}(h^*_1(X) = 1 )$. W.L.O.G. assume $\alpha \geq \beta$
\State For $(x,a)$, randomly sample $s$ from the uniform distribution $U(0,1)$
\State Construct $h^*_{\text{Fair}}$ as follows:

\Statex $h_{\text {Fair }}^*(x, a):= \begin{cases}
a=0: & \begin{cases}0 & \text{if } h_0^*(x)=0 \text{ or } h_0^*(x)=1 \text{ and } s>\frac{\alpha+\beta}{2 \alpha} \\
1 & \text{if } h_0^*(x)=1 \text{ and } s \leq \frac{\alpha+\beta}{2 \alpha}\end{cases} \\
a=1: & \begin{cases}0 & \text{if } h_1^*(x)=0 \text{ and } s>\frac{\alpha-\beta}{2(1-\beta)} \\
1 & \text{if } h_1^*(x)=1 \text{ or } h_1^*(x)=0 \text{ and } s \leq \frac{\alpha-\beta}{2(1-\beta)}\end{cases}
\end{cases}$

\Statex \textbf{return} $h^*_{\text{Fair}}$
\end{algorithmic}
\end{algorithm}

\begin{theorem}\label{algorithm1guarantees}
The classifier $h^*_{\text{Fair}}$ constructed by Algorithm~\ref{algorithm1} satisfies perfect statistical parity ($\Delta_{SP}(h^*_{\text{Fair}})=0$) and is optimal in terms of the sum of prediction changes compared to classifiers $h_0^*$ and $h_1^*$, that is
\begin{equation*}
\begin{split}
    \mathbb{P}_{\mu^X_0}({h^*_{\text{Fair}}}_0(X) \neq h^*_0(X))+\mathbb{P}_{\mu^X_1}({h^*_{\text{Fair}}}_{1}(X) \neq h^*_1(X)) = \left| \mathbb{P}_{\mu^X_0}(h^*_0(X) = 1 ) - \mathbb{P}_{\mu^{X}_{1}}(h^*_1(X) = 1 )\right|.
\end{split}  
\end{equation*}
\end{theorem}

\subsection{Fair Post-Processing With Privacy}\label{fair_post-proc_with_privacy}

Next, we delve into a private version of Algorithm~\ref{algorithm1}.  It is worth noting that achieving $(\epsilon,0)$-DP is incompatible with non-trivial guarantees of fairness and utility (see, e.g., \cite{cummings2019compatibility, agarwal_2020} for more details). As a result, our objective is to design an $( \epsilon,\delta)$-DP version of  Algorithm~\ref{algorithm1} with comparable fairness and utility guarantees, provided that $\delta > 0$. 

First, notice that this goal is not feasible just by replacing the input classifiers $h_0^*$ and $h_1^*$ with some differentially private decoupled classifiers
that are learned by applying an existing differentially private learning method ---most notably, differentially private stochastic gradient descent (DP-SGD) \cite{abadi2016deep}--- to each demographic subgroup. 
This approach implicitly assumes privacy guarantees only for non-sensitive features, whereas we aim to guarantee privacy for both sensitive and non-sensitive features. Additionally,  note that Algorithm~\ref{algorithm1} involves some computations over the dataset (e.g., in computing $\alpha$ and $\beta$ in line 1). This therefore necessitates some changes in Algorithm~\ref{algorithm1} for constructing its differentially private version. Let $h^*_{\eps,\delta}:\mathcal{X} \times \{0,1\} \rightarrow \{0,1\}$ be a classifier guaranteeing $(\eps,\delta)$-DP with respect to all features. Then, $h^*_{\eps,\delta,0}:\mathcal{X} \rightarrow \{0,1\} $ and $h^*_{\eps,\delta,1}:\mathcal{X} \rightarrow \{0,1\} $ represent the restrictions of $h^*_{\eps,\delta}$ to $A=0$ and $A=1$, respectively. 
%
%
Thus, our goal can be formulated as follows:  Given classifiers $h^*_{\epsilon,\delta,0}$ and $h^*_{\epsilon,\delta,1}$ that are $(\eps, \delta)$-DP, we wish to generate a fair classifier $h^*_{\eps',\delta',\text{Fair}}$ with the following properties: 
\begin{enumerate}
    \item $h^*_{\eps',\delta',\text{Fair}}$ is $(\eps', \delta')$-DP with some $\eps'$ and $\delta'$ (depending on $\eps$ and $\delta$),
    \item  $h^*_{\eps',\delta',\text{Fair}}$ satisfies $\gamma$-statistical parity with $\gamma$ being a positive value close to zero with high probability and in expectation,
    \item 
    Among all classifiers satisfying the same level of statistical parity gap as $h^*_{\eps',\delta',\text{Fair}}$, the classifier $h^*_{\eps',\delta',\text{Fair}}$ performs comparably with the optimal classifier in terms of utility, both with high probability and in expectation. Optimality here is defined based on minimizing the total number of prediction changes across the distributions $\mu_0^X$ and $\mu_1^X$, relative to the predictions made by $h^*_{\epsilon,\delta,0}$ and $h^*_{\epsilon,\delta,1}$.
    
    
\end{enumerate}

To this goal, we present Algorithm \ref{algorithm2} for learning such $h^*_{\eps',\delta',\text{Fair}}$, assuming that the number of data points belonging to each subgroup in a dataset is publicly known. For any dataset $D$, $\theta$ denotes the proportion of data points with $A=0$, while $\bar{\theta} = 1-\theta$ represents the proportion of data points with $A=1$.
Recall that Algorithm~\ref{algorithm1} requires estimates of the proportions of data points in each subgroup classified as label one (denoted by $\alpha$ and $\beta$). Algorithm~\ref{algorithm2} is designed to privately estimate these quantities by employing the Laplace mechanism, whose privacy guarantee is well-understood. The following theorem delineates the performance of Algorithm~\ref{algorithm2} in terms of its achievable privacy and fairness guarantees as well as bounds on its utility. 
  \begin{algorithm}[t]
\caption{Private and Fair Binary Classifier with Utility Gap Guarantee}
\label{algorithm2}
\begin{algorithmic}[1] 

\Statex \textbf{Input:} Classifiers $h^*_{\epsilon,\delta,0}$ and $h^*_{\epsilon,\delta,1}$, dataset $D=(X_i,A_i,Y_i)_{i=1}^{n}$ with $\theta n$ individuals where $A_i=0$ and $\bar{\theta} n$ individuals where $A_i=1$, privacy parameters $\eps_0$ and $\eps_1$
\Statex \textbf{Output:} Classifier $h^*_{\eps',\delta',\text{Fair}}:\mathcal{X}\times\{0,1\}\rightarrow\{0,1\}$ with $\eps' = \eps + \eps_0 + \eps_1$ and $\delta' = \delta$

\State Find $\bar{\alpha} = \frac{1}{n \theta}\sum\limits_{\substack{i=1\\ A_i = 0}}^{n} h^*_{\epsilon,\delta,0}(X_i)$ and $\bar{\beta} = \frac{1}{n \bar{\theta}}\sum\limits_{\substack{i=1\\ A_i = 1}}^{n} h^*_{\epsilon,\delta,1}(X_i)$ 

\State Add noise to $\bar{\alpha}$ and $\bar{\beta}$: 
\Statex 
Sample $l_0$ from $\text{Lap}\big(\frac{1}{n \theta \eps_0}\big)$ and $l_1$ from $\text{Lap}\big(\frac{1}{n \bar{\theta} \eps_1}\big)$.
Define $\tilde{\alpha} = \left[\bar{\alpha} + l_0 \right]_{0}^{1}$ and $\tilde{\beta} = \left[ \bar{\beta} + l_1\right]_{0}^{1}$, where $\left[\cdot\right]_{0}^{1}$ denotes the projection onto $[0,1]$.

\State For $(x,a)$, randomly sample $s$ from the uniform distribution $U(0,1)$

\State Construct $h^*_{\eps',\delta',\text{Fair}}$ as follows:
\Statex if $\tilde{\alpha} \geq \tilde{\beta}$, then 
\Statex
\[h^*_{\eps',\delta',\text{Fair}}(x, a):= 
\begin{cases}
    \mathds{1}[h^*_{\epsilon,\delta,0}(x)=1] \mathds{1}[s \leq \frac{\tilde{\alpha}+\tilde{\beta}}{2\tilde{\alpha}}] & \text{if } a = 0\\
    \mathds{1}[h^*_{\epsilon,\delta,1}(x)=0] \mathds{1}[s \leq \frac{\tilde{\alpha}-\tilde{\beta}}{2(1-\tilde{\beta})}]  + \mathds{1}[h^*_{\epsilon,\delta,1}(x)=1] & \text{if } a = 1
\end{cases}
\]

\Statex if $\tilde{\alpha} < \tilde{\beta}$, then  
\Statex

\[ h^*_{\eps',\delta',\text{Fair}}(x, a):= 
\begin{cases}
    \mathds{1}[h^*_{\epsilon,\delta,1}(x)=1] \mathds{1}[s \leq \frac{\tilde{\alpha}+\tilde{\beta}}{2 \tilde{\beta}}] & \text{if } a = 1 \\
    \mathds{1}[h^*_{\epsilon,\delta,0}(x)=0] \mathds{1}[s \leq \frac{\tilde{\beta}-\tilde{\alpha}}{2(1-\tilde{\alpha})}]  + \mathds{1}[h^*_{\epsilon,\delta,0}(x)=1] & \text{if } a = 0
\end{cases}
\]

\Statex \textbf{return} $h^*_{\eps',\delta',\text{Fair}}$ 
\end{algorithmic}
\end{algorithm}

\begin{theorem}\label{alg2WHP}
The classifier $h^*_{\eps',\delta',\text{Fair}}:\mathcal{X}\times \{0,1\}\rightarrow\{0,1\}$ constructed by Algorithm \ref{algorithm2} satisfies the following three properties:
\begin{itemize}
    \item (Privacy guarantee) $h^*_{\eps',\delta',\text{Fair}}$ satisfies $(\eps',\delta')$-DP with $\eps' = \eps + \eps_0 + \eps_1$ and $\delta' = \delta$,
    \item (Fairness guarantee) We have:
\begin{equation*}
    \Delta_{SP}(h^*_{\eps',\delta',\text{Fair}})  \leq_{\eta} \frac{\log(4/\eta)}{n \theta \eps_0} + \frac{\log(4/\eta)}{n \bar{\theta} \eps_1}+\sqrt{\log(\frac{8}{\eta})\frac{1}{2n \theta}}+\sqrt{\log(\frac{8}{\eta})\frac{1}{2 n \bar{\theta}}},
\end{equation*}
\item (Utility guarantee) Let $\text{err}^*(h^*_{\epsilon,\delta,0},h^*_{\epsilon,\delta,1})$ be defined as:
\begin{equation}
\label{Def:err}
\setlength{\jot}{-0.4cm}
\begin{split}
    \min_{\substack{\hat{h}: \mathcal{X}\times\{0,1\}\rightarrow \{0,1\} \\ \Delta_{SP}(\hat{h})\leq \Delta_{SP}(h^*_{\eps',\delta',\text{Fair}})}} \biggl[ 
    \mathbb{P}_{\mu^X_0}(\hat{h}_{0}(X) \neq h^*_{\epsilon,\delta,0}(X))  + \mathbb{P}_{\mu^X_1}(\hat{h}_{1}(X) \neq h^*_{\epsilon,\delta,1}(X)) \biggr].
\end{split}
\end{equation}
Then, we have:
\begin{equation*}
    \begin{split}
        \mathbb{P}_{\mu^X_0}({h^*_{\eps',\delta',\text{Fair}}}_{0}(X) \neq h^*_{\epsilon,\delta,0}(X)) &+\mathbb{P}_{\mu^X_1}({h^*_{\eps',\delta',\text{Fair}}}_{1}(X) \neq h^*_{\epsilon,\delta,1}(X)) \\
        &  \hspace{-1.8cm} \leq_{\eta}\text{err}^*(h^*_{\epsilon,\delta,0},h^*_{\epsilon,\delta,1}) + \frac{5}{2}\Big( \frac{\log(4/\eta)}{n \theta \eps_0} + \frac{\log(4/\eta)}{n \bar{\theta} \eps_1} +\sqrt{\log(\frac{8}{\eta})\frac{1}{2n \theta}}+\sqrt{\log(\frac{8}{\eta})\frac{1}{2 n \bar{\theta}}} \Big). 
    \end{split}
\end{equation*}
\end{itemize}
\end{theorem}
The privacy guarantee of $h^*_{\eps',\delta',\text{Fair}}$ has two components: the privacy guarantees of the input classifiers and the Laplace mechanism employed to privately estimate $\bar\alpha$ and $\bar\beta$ (line 2 in Algorithm~\ref{algorithm2}). Thus, the privacy guarantee in the above theorem follows directly from a composition result (e.g., basic composition \cite{dwork2014algorithmic}).   
However, the analysis pertaining to fairness and utility guarantees is rather long and thus deferred to the appendix. The fairness guarantee demonstrates, as expected, that perfect statistical parity can no longer be achievable when requiring privacy as well. Nevertheless, the theorem shows that $h^*_{\eps',\delta',\text{Fair}}$ satisfies $\gamma$-statistical parity where $\gamma>0$ is a small constant provided that the dataset is not highly imbalanced (i.e., $\theta$ close to $0$ or $1$) and $n$ is sufficiently large (compared to $1/\eps_0$  and $1/\eps_1$).    
We remark that these assumptions were implicitly made in \cite{inherent}, in which they ignored the error in estimating the proportion of label one in each subgroup. Finally, the utility guarantee presented in the theorem indicates that the necessary perturbations in the final prediction by $h^*_{\eps',\delta',\text{Fair}}$ is almost identical to what is expected by the optimal classifier having the same level of statistical parity.

Rather than aiming to achieve a small statistical parity gap and a small utility gap with high probability, we could alternatively focus on the \textit{average} guarantees for fairness and utility, as expounded by the next result.  

\begin{proposition}\label{alg2Exp}
The classifier $h^*_{\eps',\delta',\text{Fair}}:\mathcal{X}\times \{0,1\}\rightarrow\{0,1\}$ constructed by Algorithm \ref{algorithm2} satisfies the following two properties:
\begin{itemize}
    \item (Fairness guarantee) $\mathbb{E}\big[\Delta_{SP}(h^*_{\eps',\delta',\text{Fair}})\big]  \leq \frac{1}{n \theta \eps_0} + \frac{1}{n \bar{\theta} \eps_1}+ \sqrt{\frac{1}{4 n \theta}} + \sqrt{\frac{1}{4 n \bar{\theta}}},$
    \item (Utility guarantee) We have 
    \begin{align*}
        & \mathbb{E}\Bigl[ \mathbb{P}_{\mu^X_0}({h^*_{\eps',\delta',\text{Fair}}}_{0}(X) \neq h^*_{\epsilon,\delta,0}(X))+\mathbb{P}_{\mu^X_1}({h^*_{\eps',\delta',\text{Fair}}}_{1}(X) \neq h^*_{\epsilon,\delta,1}(X))  \Bigr]  \\ & \hspace{5.5cm} \leq \mathbb{E}\bigl[ \text{err}^*(h^*_{\epsilon,\delta,0},h^*_{\epsilon,\delta,1})\bigr] + \frac{5}{2}\Big(\frac{1}{n \theta \eps_0} + \frac{1}{n \bar{\theta} \eps_1} + \sqrt{\frac{1}{4 n \theta}} + \sqrt{\frac{1}{4 n \bar{\theta}}} \Big),
\end{align*}
where $\text{err}^*(\cdot,\cdot)$ was defined in \eqref{Def:err}.
\end{itemize}

\end{proposition}

\section{Experiments}\label{experiments}

In this section, we seek to empirically compare Algorithm~\ref{algorithm2} with the current state-of-the-art differentially private fair classifier, namely, DP-FERMI \cite{lowy2023stochastic}. To this goal, we focus on two benchmark datasets from the UCI machine learning repository \cite{UCI}: the Adult and Credit Card datasets, both with binary sensitive attributes and target labels. 

In our experiments, we assessed two privacy configurations: $(\eps' = 3, \delta' = 10^{-5})$ and $(\eps' = 9, \delta' = 10^{-5})$. When applying Algorithm 2 to the Adult and Credit Card datasets, we set $\eps_0 = \eps_1 = 0.05$ and $\eps_0 = \eps_1 = 0.1$, respectively. We then feed Algorithm~\ref{algorithm2} with the decoupled classifiers, $h^*_{\epsilon,\delta,0}$ and $h^*_{\epsilon,\delta,1}$ with parameters $(\eps,\delta)$ chosen in a way as to satisfy $\eps' = \eps + \eps_0 + \eps_1$ and $\delta' = \delta $ for each specified $(\eps', \delta')$ pair. We trained these classifiers via DP-SGD.
The training was conducted using Opacus \cite{Opacus}, an open-source PyTorch library designed for training deep learning models with differential privacy. Accordingly, we chose the standard deviations of Gaussian noise in DP-SGD to be $3.13$ and $1.5$ for the Adult dataset, and to $4.44$ and $2.49$ for the Credit Card dataset, in order to achieve their respective privacy parameters $(\eps' = 3, \delta' = 10^{-5})$ and $(\eps' = 9, \delta' = 10^{-5})$. Across all experiments of DP-SGD, we consistently applied a clipping constant of $1.5$ and a learning rate of $0.01$. To achieve the most precise computation of privacy parameters, we utilized the PRV accountant \cite{gopi2021numerical}, the state-of-the-art accounting method, to determine the values for $(\eps,\delta)$ for the classifiers $h^*_{\epsilon,\delta,0}$ and $h^*_{\epsilon,\delta,1}$. Additionally, to evaluate the variance in accounting methodologies, we conducted privacy parameter computations using two alternative methods: moments accountant method \cite{abadi2016deep} and GDP accountant \cite{dong2019gaussian,bu2020deep}, as detailed in the appendix.

 
 For comparability with DP-FERMI, we employed logistic regression models. DP-FERMI utilizes a loss function with a regularization constant $\lambda$ to limit statistical parity violations. A higher $\lambda$ imposes stricter penalties for fairness violations, often at the cost of reduced accuracy. In this framework, the desired $\epsilon'$ and $\delta'$ are set and the required noise levels to meet these privacy guarantees are computed. It is important to remark that DP-FERMI implicitly assumes that both the fraction of data points in the minority subgroup and the size of the entire dataset are publicly available ---similar to the assumption made in our work.
 All models, including Algorithm 2 and DP-FERMI, were trained for 50 epochs, except the Credit Card models using $(\eps'=9,\delta'=10^{-5})$ parameters, which were trained for 100 epochs. Within DP-FERMI, learning rates were set at $\eta_{\theta}=0.005$ and $\eta_{W}=0.01$. We chose these parameters mainly because they were empirically shown in \cite{lowy2023stochastic} to be optimal for the datasets under consideration.  
 A uniform batch size of 1024 was maintained for all experiments.

While our theoretical framework focuses on the sum of prediction changes across subgroup distributions, for comparison purposes with DP-FERMI, we used overall accuracy as a utility metric. The final results are presented in Tables \ref{tab:Adult_Eps_3_PRV}, \ref{tab:Adult_Eps_9_PRV}, \ref{tab:Credit_Eps_3_PRV}, and \ref{tab:Credit_Eps_9_PRV}. All of the results are averages over 10 trials. We remark that DP-FERMI offers two privacy options: one for sensitive attributes and another for all features. We used noise parameters for all-feature privacy in DP-FERMI to compare fairly with Algorithm \ref{algorithm2}. 

\begin{table}[t]
\centering
\renewcommand{\arraystretch}{1.2} 
\begin{minipage}{0.5\textwidth}
  \centering
  \caption{Adult Dataset ($\eps' = 3, \delta' = 10^{-5}$) - PRV Accountant}
  \label{tab:Adult_Eps_3_PRV}
  \begin{tabular}{|
  >{\centering\arraybackslash}m{0.4\columnwidth}|
  >{\centering\arraybackslash}m{0.2\columnwidth}|
  >{\centering\arraybackslash}m{0.2\columnwidth}|}
    \hline
    Method & Accuracy & Statistical Parity Gap \\
    \hline
    Algorithm 2 & 0.7763 & 0.0074\\
    \hline
    DP-FERMI ($\lambda=0.5$) & 0.7998 & 0.1020\\
    \hline
    DP-FERMI ($\lambda=1$) & 0.7859 & 0.0462\\
    \hline
    DP-FERMI ($\lambda=1.5$) & 0.7822 & 0.0267\\
    \hline
    DP-FERMI ($\lambda=1.8$) &  0.7770 & 0.0182\\
    \hline
    DP-FERMI ($\lambda=2.5$) & 0.7673 & 0.0099\\
    \hline
  \end{tabular}
\end{minipage}%
\begin{minipage}{0.5\textwidth}
  \centering
  \caption{Adult Dataset ($\eps' = 9, \delta' = 10^{-5}$) - PRV Accountant}
  \label{tab:Adult_Eps_9_PRV}
  \begin{tabular}{|
  >{\centering\arraybackslash}m{0.4\columnwidth}|
  >{\centering\arraybackslash}m{0.2\columnwidth}|
  >{\centering\arraybackslash}m{0.2\columnwidth}|}
    \hline
    Method & Accuracy & Statistical Parity Gap \\
    \hline
    Algorithm 2 & 0.7790 & 0.0091\\
    \hline
    DP-FERMI ($\lambda=0.5$) & 0.8091 & 0.0944\\
    \hline
    DP-FERMI ($\lambda=1$) & 0.7923 & 0.0413\\
    \hline
    DP-FERMI ($\lambda=1.5$) & 0.7810 & 0.0152\\
    \hline
    DP-FERMI ($\lambda=1.7$) & 0.7782 & 0.0121\\
    \hline
    DP-FERMI ($\lambda=2.5$) & 0.7693 & 0.0030\\
    \hline
  \end{tabular}
\end{minipage}
\end{table}

\begin{table}[t]
\centering
\renewcommand{\arraystretch}{1.2} 
\begin{minipage}{0.5\textwidth}
  \centering
  \caption{Credit Card Dataset ($\eps' = 3, \delta' = 10^{-5}$) - PRV Accountant}
  \label{tab:Credit_Eps_3_PRV}
  \begin{tabular}{|
  >{\centering\arraybackslash}m{0.4\columnwidth}|
  >{\centering\arraybackslash}m{0.2\columnwidth}|
  >{\centering\arraybackslash}m{0.2\columnwidth}|}
    \hline
    Method & Accuracy & Statistical Parity Gap \\
    \hline
    Algorithm 2 & 0.7844 & 0.0086\\
    \hline
    DP-FERMI ($\lambda=0.1$) & 0.7899 & 0.0212\\
    \hline
    DP-FERMI ($\lambda=0.2$) & 0.7846 & 0.0193\\
    \hline
    DP-FERMI ($\lambda=0.5$) & 0.7777 & 0.0185\\
    \hline
    DP-FERMI ($\lambda=1$) & 0.7759 & 0.0105\\
    \hline
    DP-FERMI ($\lambda=2.5$) & 0.7669 & 0.0110\\
    \hline
  \end{tabular}
\end{minipage}%
\begin{minipage}{0.5\textwidth}
  \centering
  \caption{Credit Card Dataset ($\eps' = 9, \delta' = 10^{-5}$) - PRV Accountant}
  \label{tab:Credit_Eps_9_PRV}
  \begin{tabular}{|
  >{\centering\arraybackslash}m{0.4\columnwidth}|
  >{\centering\arraybackslash}m{0.2\columnwidth}|
  >{\centering\arraybackslash}m{0.2\columnwidth}|}
    \hline
    Method & Accuracy & Statistical Parity Gap \\
    \hline
    Algorithm 2 & 0.7900 & 0.0056\\
    \hline
    DP-FERMI ($\lambda=0.25$) & 0.7996 & 0.0188\\
    \hline
    DP-FERMI ($\lambda=0.3$) & 0.7971 & 0.0182\\
    \hline
    DP-FERMI ($\lambda=0.5$) & 0.7912 & 0.0172\\
    \hline
    DP-FERMI ($\lambda=1$) & 0.7895 & 0.0105\\
    \hline
    DP-FERMI ($\lambda=2.5$) & 0.7884 & 0.0066\\
    \hline
  \end{tabular}
\end{minipage}
\end{table}

 In the Adult dataset, as shown in Table \ref{tab:Adult_Eps_3_PRV}, Algorithm \ref{algorithm2} achieves accuracy 0.7763. Setting $\lambda$ at 1.8, DP-FERMI achieves a comparable accuracy (0.7770) but exhibits a statistical parity gap more than twice as large as what is guaranteed by Algorithm \ref{algorithm2}. Table \ref{tab:Adult_Eps_9_PRV} illustrates that Algorithm \ref{algorithm2} and DP-FERMI with $\lambda=1.7$ achieve similar accuracy, yet Algorithm \ref{algorithm2} exhibits a smaller statistical parity gap (0.0091) compared to DP-FERMI (0.0121). Altogether, these tables empirically underscore that applying Algorithm \ref{algorithm2} to the Adult dataset results in a better fairness guarantee while maintaining similar privacy and accuracy guarantees. 

Similar trends are observed in the Credit Card dataset. Table \ref{tab:Credit_Eps_3_PRV} displays Algorithm \ref{algorithm2} achieving an accuracy of 0.7844, closely matched by DP-FERMI with an accuracy of 0.7846 at $\lambda=0.2$. However, the statistical parity gap of Algorithm \ref{algorithm2} is less than half that of DP-FERMI (0.0086 compared to 0.0193). As shown in Table \ref{tab:Credit_Eps_9_PRV}, Algorithm \ref{algorithm2} and DP-FERMI with $\lambda = 1$ attain similar accuracy, yet Algorithm \ref{algorithm2} exhibits a statistical parity gap of 0.0056, nearly half of DP-FERMI's 0.0105. 
We provide comprehensive descriptions of the datasets, training processes, and additional results in \nameref{appendix_experiments}. 







\bibliographystyle{unsrt}  
\bibliography{references}

\appendix

This appendix is organized into two sections: \nameref{proofs}, containing proofs of our theoretical results; and \nameref{appendix_experiments}, providing in-depth information about the experimental setup, datasets used, and supplementary experiments. 

\section*{Appendix A} \label{proofs}
\begin{proof}[Proof of Proposition \ref{proposition1}]
We have classifiers $h_0^*: \mathcal{X}\rightarrow \{0,1\}$ and $h_1^*: \mathcal{X}\rightarrow \{0,1\}$ trained on subgroups specified by the sensitive attribute $A$ with values 0 and 1, respectively. We can combine classifiers $h_0^*$ and $h_1^*$ to have a group aware classifier $h^*:\mathcal{X} \times \{0,1\} \rightarrow \{0,1\}$. Basically, $h^*(x,0)=h_0^*(x) \hspace{0.2cm}\forall x\in\mathcal{X}$ and $h^*(x,1)=h_1^*(x) \hspace{0.2cm} \forall x\in\mathcal{X}$. Let $Y^*=h^*(X,A)$ and $\hat{Y} = \hat{h}(X,A)$. Then for $a \in \{0,1\}$, we have:
\begin{align}
    d_{TV}(\mu_a(Y^*),\mu_a(\hat{Y})) &= \left| \mu_a(Y^* = 1) - \mu_a(\hat{Y} = 1)\right| \nonumber \\
    &= \left| \mathbb{E}_{\mu_a^X}[h^*_a(X)] - \mathbb{E}_{\mu_a^X}[\hat{h}_a(X)] \right| \nonumber \\
    &\leq \mathbb{E}_{\mu_a^X} \left[ \left| h^*_a(X) - \hat{h}_a(X) \right| \right] \nonumber \\
    &= \mathbb{P}_{\mu_a^X}\left( Y^* \neq \hat{Y}\right). \label{eq:tv_bound}
\end{align}
Therefore, from (\ref{eq:tv_bound}) it follows:
\begin{equation*}
    d_{TV}(\mu_a(Y^*),\mu_a(\hat{Y})) \leq \mathbb{P}_{\mu_a^X}\left( Y^* \neq \hat{Y}\right).
\end{equation*}
We assumed $\hat{Y}=\hat{h}(X,A)$ satisfies $\gamma$ statistical parity ($\Delta_{SP}(\hat{h}) \leq \gamma$). Thus, we have:
\begin{align*}
    d_{TV}(\mu_0(\hat{Y}),\mu_1(\hat{Y})) = \left| 
    \mu_0(\hat{Y}=1) - \mu_1(\hat{Y}=1)\right|  = \Delta_{SP}(\hat{h}) \leq \gamma. 
\end{align*}
Since $d_{TV}(\cdot,\cdot)$ is symmetric and satisfies the triangle inequality, we have:
\begin{align}
    d_{TV}(\mu_0(Y^*),\mu_1(Y^*))  
    &\leq d_{TV}(\mu_0(Y^*),\mu_0(\hat{Y})) + d_{TV}(\mu_0(\hat{Y}),\mu_1(\hat{Y})) + d_{TV}(\mu_1(\hat{Y}),\mu_1(Y^*)) \nonumber \\
    &\leq d_{TV}(\mu_0(Y^*),\mu_0(\hat{Y})) + \gamma  + d_{TV}(\mu_1(Y^*),\mu_1(\hat{Y})). \label{eq:triangle_inequality}
\end{align}
Combining (\ref{eq:triangle_inequality})  with (\ref{eq:tv_bound}),  we have:
\begin{align*}
    d_{TV}(\mu_0(Y^*),\mu_1(Y^*)) & \leq \mathbb{P}_{\mu_0^X}\left( Y^* \neq \hat{Y}\right) + \mathbb{P}_{\mu_1^X}\left( Y^* \neq \hat{Y}\right) + \gamma \\
    & = \mathbb{P}_{\mu^X_0}(\hat{h}_0(X) \neq h^*_0(X)) + \mathbb{P}_{\mu^X_1}(\hat{h}_1(X) \neq h^*_1(X)) + \gamma.
\end{align*}
Therefore, we can obtain:
\begin{align*}
     \Bigl| & \mathbb{P}_{\mu_0^X}(h_0^*(X)=1) - \mathbb{P}_{\mu_1^X}(h_1^*(X)=1)\Bigr|   \leq  \mathbb{P}_{\mu^X_0}(\hat{h}_0(X) \neq h^*_0(X)) + \mathbb{P}_{\mu^X_1}(\hat{h}_1(X) \neq h^*_1(X)) + \gamma.
\end{align*}
Thus, we have:
\begin{equation*}
    \begin{split}
        & \mathbb{P}_{\mu^X_0}(\hat{h}_0(X) \neq h^*_0(X))+\mathbb{P}_{\mu^X_1}(\hat{h}_1(X) \neq h^*_1(X)) \geq  \left| \mathbb{P}_{\mu^X_0}(h^*_0(X) = 1 ) - \mathbb{P}_{\mu^{X}_{1}}(h^*_1(X) = 1 )\right| - \gamma.
    \end{split}
\end{equation*}
which concludes the proof of Proposition \ref{proposition1}.
\end{proof}

\begin{proof}[Proof of Theorem \ref{algorithm1guarantees}] 
Let $s$ in Algorithm \ref{algorithm1} be the realization of the random variable $S$. 

\begin{align*}
    &\mathbb{P}_{\mu}\left( h^*_{\text{Fair}}(X,A) = 1 | A=0\right) = \mathbb{P}_{\mu_0^X}(h_0^*(X)=1) \mathbb{P}(S \leq \frac{\alpha+\beta}{2\alpha}) =  \alpha  \left(\frac{\alpha+\beta}{2\alpha}\right) = \frac{\alpha + \beta}{2}.
\end{align*}
\begin{align*}
    \mathbb{P}_{\mu}\left( h^*_{\text{Fair}}(X,A) = 1 | A=1\right) & = 
     \mathbb{P}_{\mu_1^X}(h_1^*(X)=1)  + \mathbb{P}_{\mu_1^X}(h_1^*(X)=0)\mathbb{P}\left(S \leq \frac{\alpha-\beta}{2(1-\beta)}\right)  \\
     & = \beta + (1-\beta)\left(\frac{\alpha-\beta}{2(1-\beta)}\right)\\
     &= \frac{\alpha + \beta}{2}.
\end{align*}
We have:
\begin{equation*}
    \Delta_{SP}(h^*_{\text{Fair}}) = \left|\frac{\alpha + \beta}{2} - \frac{\alpha + \beta}{2}\right| = 0.
\end{equation*}
Therefore, perfect statistical parity is satisfied. Now we show that $h^*_{\text{Fair}}$ is optimal.
\begin{align*}
&\mathbb{P}_{\mu^X_0}({h^*_{\text{Fair}}}_{0}(X) \neq h^*_0(X)) = \mathbb{P}_{\mu_0^X}(h^*_0(X) = 1) \mathbb{P}\left(S > \frac{\alpha + \beta}{2 \alpha}\right) = \alpha \left(\frac{\alpha-\beta}{2\alpha}\right) = \frac{\alpha-\beta}{2} .
\end{align*}
\begin{align*}
\mathbb{P}_{\mu^X_1}({h^*_{\text{Fair}}}_{1}(X) \neq h^*_1(X)) & = \mathbb{P}_{\mu_1^X}(h^*_1(X) = 0) \mathbb{P} \left(S \leq \frac{\alpha - \beta}{2 (1-\beta)} \right) \\ & = (1-\beta)\left(\frac{\alpha - \beta}{2 (1-\beta)}\right) = \frac{\alpha-\beta}{2} .
\end{align*}
Thus, we have:
\begin{equation*}
    \begin{split}
        \mathbb{P}_{\mu^X_0}({h^*_{\text{Fair}}}_{0}(X) \neq h^*_0(X))+\mathbb{P}_{\mu^X_1}({h^*_{\text{Fair}}}_{1}(X) \neq h^*_1(X))  & =  \frac{\alpha-\beta}{2} + \frac{\alpha-\beta}{2}  = \alpha - \beta = \left| \alpha - \beta \right|  \\
    & = \left|\mathbb{P}_{\mu^X_0} (h^*_0(X) = 1 ) - \mathbb{P}_{\mu^X_1}(h^*_1(X) = 1 )\right|.
    \end{split}
\end{equation*}
Therefore, $h^*_{\text{Fair}}$ satisfies perfect statistical parity and attains the utility lower bound of Proposition \ref{proposition1}. 
\end{proof}

\begin{proof}[Proof of Theorem \ref{alg2WHP}]
\mbox{}
 \begin{itemize}

    \item 
    Every time we access the dataset to compute a query, it is essential to account for the privacy budget being consumed. Algorithm \ref{algorithm2} accesses the training dataset to train the classifier $h^*_{\epsilon,\delta}: \mathcal{X} \times \{0,1\} \rightarrow \{0,1\}$. Additionally, computing $\tilde{\alpha}$ and $\tilde{\beta}$ involves using the post-processing dataset. Learning $h^*_{\epsilon,\delta}: \mathcal{X} \times \{0,1\} \rightarrow \{0,1\}$, from which we derive two classifiers $h^*_{\epsilon,\delta,0}$ and $h^*_{\epsilon,\delta,1}$, was conducted with privacy parameters $(\epsilon,\delta)$. Computing $\tilde{\alpha}$ and $\tilde{\beta}$ utilizes the Laplace mechanism followed by post-processing (projection). Hence, these two mechanisms will satisfy $\epsilon_0$-DP and $\epsilon_1$-DP \cite{dwork2014algorithmic}. By basic composition, we can conclude that $h^*_{\eps',\delta',\text{Fair}}$ satisfies $(\eps',\delta')$-DP with $\eps' = \epsilon + \epsilon_0 + \epsilon_1$ and $\delta' = \delta$. Note that we assume the number of data points belonging to each subgroup in a dataset is public knowledge. Specifically, we assume that $\theta n$ and $\bar{\theta} n$ are publicly known, and thus computing them does not consume any privacy budget. \\
    
    \item 
    Let $\alpha = \mathbb{P}_{\mu^X_0} (h^*_{\epsilon,\delta,0}(X) = 1 )$ and $\beta = \mathbb{P}_{\mu^X_1}(h^*_{\epsilon,\delta,1}(X) = 1 )$. Let $\bar{\alpha} = \alpha + e_0$ and $\bar{\beta} = \beta + e_1$.
    To prove the claims in the theorem, we assume $\tilde{\alpha} \geq \Tilde{\beta}$. For the case that $\tilde{\alpha} < \Tilde{\beta}$, we will have the same results by symmetry. 
    Also, let $L_0 \sim \text{Lap}\big(\frac{1}{n \theta \eps_0}\big)$ and $L_1 \sim \text{Lap}\big(\frac{1}{n \bar{\theta} \eps_1}\big)$. In Algorithm \ref{algorithm2}, we sample $l_0$ from $L_0$ and $l_1$ from $L_1$. Similar to the proof of Theorem \ref{algorithm1guarantees}, let $s$ in Algorithm \ref{algorithm2} be the realization of the random variable $S$. 
    By definition, we have: 
\begin{align*}
    \Delta_{SP}(h^*_{\eps',\delta',\text{Fair}}) =  \bigl| &  \mathbb{P}_{\mu_0^X} \left(h^*_{\eps',\delta',\text{Fair}}(X,0) =1 \right)  -  \mathbb{P}_{\mu_1^X} \left(h^*_{\eps',\delta',\text{Fair}}(X,1) =1 \right) \bigr| .
\end{align*}
We first compute $\mathbb{P}_{\mu_0^X} \left(h^*_{\eps',\delta',\text{Fair}}(X,0) =1 \right)$:
\begin{align*}
\mathbb{P}_{\mu_0^X} \left(h^*_{\eps',\delta',\text{Fair}}(X,0) =1 \right) \nonumber & =  \mathbb{P}_{\mu^X_0} (h^*_{\epsilon,\delta,0}(X) = 1 ) \mathbb{P}\left(S \leq \frac{\tilde{\alpha}+\tilde{\beta}}{2\tilde{\alpha}}\right) \\ 
& = \alpha  \frac{\tilde{\alpha}+\tilde{\beta}}{2\tilde{\alpha}} \\
&= \frac{\alpha}{\tilde{\alpha}}\left(\frac{\tilde{\alpha}+\tilde{\beta}}{2}\right) . 
\end{align*}
We then compute $\mathbb{P}_{\mu_1^X} \left(h^*_{\eps',\delta',\text{Fair}}(X,1) =1 \right)$:
\begin{align*}
 \mathbb{P}_{\mu_1^X} \left(h^*_{\eps',\delta',\text{Fair}}(X,1) =1 \right) & = \mathbb{P}_{\mu^X_1} (h^*_{\epsilon,\delta,1}(X) = 1 ) + \nonumber   \mathbb{P}_{\mu^X_1} (h^*_{\epsilon,\delta,1}(X) = 0 ) \mathbb{P}\left(S \leq \frac{\tilde{\alpha}-\tilde{\beta}}{2(1-\tilde{\beta})}\right)   \\
 & = \beta + (1-\beta)\left( \frac{\tilde{\alpha}-\tilde{\beta}}{2(1-\tilde{\beta})}\right) = \beta + \frac{(1-\beta)}{(1-\tilde{\beta})} \left( \frac{\tilde{\alpha}-\tilde{\beta}}{2}\right).
\end{align*}

For each realization of $\tilde{\alpha}$ and $\Tilde{\beta}$, let $\bar{\alpha} = \alpha + e_0$ , $\tilde{\alpha}= \bar{\alpha} + d_0$,  $\bar{\beta} = \beta + e_1$ , $\tilde{\beta}= \bar{\beta} + d_1$. Thus, we have:
\begin{align}\label{eq:statistical_parity_upper_bound}
      \Delta_{SP}(h^*_{\eps',\delta',\text{Fair}})  & =  \left| \frac{\alpha}{\tilde{\alpha}}\left(\frac{\tilde{\alpha}+\tilde{\beta}}{2}\right)  - \left[\beta + \frac{(1-\beta)}{(1-\tilde{\beta})} \left( \frac{\tilde{\alpha}-\tilde{\beta}}{2}\right)\right] \right|\nonumber\\  
    & = \Biggl| \frac{\tilde{\alpha}- e_0 - d_0}{\tilde{\alpha}}\left(\frac{\tilde{\alpha}+\tilde{\beta}}{2}\right) -  \Biggl[\tilde{\beta}-e_1-d_1 +\frac{(1-\tilde{\beta}+e_1+d_1)}{(1-\tilde{\beta})} \left( \frac{\tilde{\alpha}-\tilde{\beta}}{2}\right)\Biggr] \Biggr|  \nonumber\\
    & = \Biggl| \left(\frac{\tilde{\alpha}+\tilde{\beta}}{2}\right) - (e_0+d_0)\left(\frac{\tilde{\alpha}+\tilde{\beta}}{2\tilde{\alpha}}\right) -  \tilde{\beta} + e_1+d_1 - \left( \frac{\tilde{\alpha}-\tilde{\beta}}{2}\right) - (e_1+d_1) \left( \frac{\tilde{\alpha} - \tilde{\beta}}{2(1-\tilde{\beta})}\right)\Biggr|  \nonumber\\
    & =  \left| (e_1+d_1) \left( 1 - \frac{\tilde{\alpha} - \tilde{\beta}}{2(1-\tilde{\beta})}\right) - (e_0+d_0) \left( \frac{\tilde{\alpha}+\tilde{\beta}}{2\tilde{\alpha}}\right)\right|    \nonumber\\
    & \leq \left| (e_1+d_1) \left( 1 - \frac{\tilde{\alpha} - \tilde{\beta}}{2(1-\tilde{\beta})}\right) \right| + \left| (e_0+d_0) \left( \frac{\tilde{\alpha}+\tilde{\beta}}{2\tilde{\alpha}}\right)\right|  \nonumber\\
    &  \leq \left| (e_1+d_1)  \right| +  \left| (e_0+d_0)  \right| \nonumber\\
    & \leq |e_0|+|e_1| + |d_0|+|d_1|.
\end{align}
The last line follows from the fact that $0\leq \Tilde{\alpha} \leq 1$, $0\leq \Tilde{\beta} \leq 1$, and $\tilde{\alpha} \geq \tilde{\beta}$.

We know that  $ |e_0| =   | \bar{\alpha} - \alpha | = \left|\frac{1}{n \theta}\sum\limits_{\substack{i=1\\ A_i = 0}}^{n} h^*_{\epsilon,\delta,0}(X_i) -  \mathbb{P}_{\mu^X_0} (h^*_{\epsilon,\delta,0}(X) = 1 ) 
  \right|$. 
  From Hoeffding's inequality, we know that if $X_1,X_2,\ldots,X_n$ are i.i.d. random variables in $[0,1]$, then:
  \begin{equation*}
      \mathbb{P}\left[ \left| \frac{1}{n} \sum_{i=1}^n X_i - \mathbb{E}[X]  \right| \geq t \right] \leq 2 e^{-2nt^2}.
  \end{equation*}

Therefore, we can conclude that: 
 \begin{equation*}
      \mathbb{P}\left[ \left| \frac{1}{n \theta}\sum\limits_{\substack{i=1\\ A_i = 0}}^{n} h^*_{\epsilon,\delta,0}(X_i) -  \mathbb{P}_{\mu^X_0} (h^*_{\epsilon,\delta,0}(X) = 1 )  \right| \geq t \right] \leq 2 e^{-2n\theta t^2},
  \end{equation*}
Which means:
\begin{equation*}
      \mathbb{P}\left[ \left| \frac{1}{n \theta}\sum\limits_{\substack{i=1\\ A_i = 0}}^{n} h^*_{\epsilon,\delta,0}(X_i) -  \mathbb{P}_{\mu^X_0} (h^*_{\epsilon,\delta,0}(X) = 1 )  \right| \geq \sqrt{\frac{1}{2n \theta}\log(\frac{2}{\eta})} \right] \leq \eta.
  \end{equation*}

  Thus, we have:
  \begin{equation}\label{eq:approximation error bound e0}
      |e_0| \leq_{\eta} \sqrt{\frac{1}{2n \theta}\log(\frac{2}{\eta})} .
  \end{equation}
  Similarly, it follows that:
\begin{equation}\label{eq:approximation error bound e1}
      |e_1| \leq_{\eta} \sqrt{\frac{1}{2n \bar{\theta}}\log(\frac{2}{\eta})}. 
  \end{equation}
  
Let $l_0$ and $l_1$ be the realization of the Laplace noises in Algorithm \ref{algorithm2}. We know that $|d_0| \leq |l_0|$ and $|d_1| \leq |l_1|$ since $d_0$ and $d_1$ are computed after projection. 
On the other hand, from \cite{dwork2014algorithmic}, we know that if $L \sim \text{Lap}(\frac{\Delta_1^q}{ \epsilon})$, then:
\begin{equation*}
    \mathbb{P}\left[|L|\geq\left(\log{\frac{1}{\eta}}\right)\left(\frac{\Delta_1^q}{\epsilon}\right)\right] \leq \eta.
\end{equation*}
where $\Delta_1^q$ is the $\ell_1$-sensitivity of the query to which we add noise.

Given $L_0 \sim \text{Lap}\left(\frac{1}{n \theta \eps_0}\right)$ and $L_1 \sim \text{Lap}\left(\frac{1}{n \bar{\theta} \eps_1}\right)$, we have:\\
$|L_0| \leq_{\eta} \left(\log{\frac{1}{\eta}}\right)\left(\frac{1}{n \theta\eps_0}\right)$ and  
$|L_1| \leq_{\eta} \left(\log{\frac{1}{\eta}}\right)\left(\frac{1}{n \bar{\theta}\eps_1}\right).$
Comparing these two inequalities with (\ref{eq:approximation error bound e0}) and (\ref{eq:approximation error bound e1}), we can conclude that with probability at least $(1-\eta)^4$, we have:
\begin{equation*}
   |L_0| + |L_1| +  |e_0| + |e_1|\leq \left[ \left(\log{\frac{1}{\eta}}\right)\left(\frac{1}{n \theta\eps_0}\right) + \left(\log{\frac{1}{\eta}}\right)\left(\frac{1}{n \bar{\theta}\eps_1}\right) + \sqrt{\frac{1}{2n \theta}\log(\frac{2}{\eta})} + \sqrt{\frac{1}{2n \bar{\theta}}\log(\frac{2}{\eta})}\right].
\end{equation*}
Since $(1-\eta)^4 \geq 1-4\eta$ for $0 \leq \eta \leq 1$, we have:
\begin{equation*}
    |L_0| + |L_1| +  |e_0| + |e_1|\leq_{4\eta} \left[ \left(\log{\frac{1}{\eta}}\right)\left(\frac{1}{n \theta\eps_0}\right) + \left(\log{\frac{1}{\eta}}\right)\left(\frac{1}{n \bar{\theta}\eps_1}\right) + \sqrt{\frac{1}{2n \theta}\log(\frac{2}{\eta})} + \sqrt{\frac{1}{2n \bar{\theta}}\log(\frac{2}{\eta})}\right].
\end{equation*}
Equivalently, we have:
\begin{equation}\label{eq:laplace_upper_bound}
    |L_0| + |L_1| +  |e_0| + |e_1|\leq_{\eta} \left[ \left(\log{\frac{4}{\eta}}\right)\left(\frac{1}{n \theta\eps_0}\right) + \left(\log{\frac{4}{\eta}}\right)\left(\frac{1}{n \bar{\theta}\eps_1}\right) + \sqrt{\frac{1}{2n \theta}\log(\frac{8}{\eta})} + \sqrt{\frac{1}{2n \bar{\theta}}\log(\frac{8}{\eta})}\right].
\end{equation}
Combining (\ref{eq:statistical_parity_upper_bound}) and (\ref{eq:laplace_upper_bound}), it can be shown:
\begin{equation*}
    \Delta_{SP}(h^*_{\eps',\delta',\text{Fair}})  \leq_{\eta} \left[ \left(\log{\frac{4}{\eta}}\right)\left(\frac{1}{n \theta\eps_0}\right) + \left(\log{\frac{4}{\eta}}\right)\left(\frac{1}{n \bar{\theta}\eps_1}\right) + \sqrt{\frac{1}{2n \theta}\log(\frac{8}{\eta})} + \sqrt{\frac{1}{2n \bar{\theta}}\log(\frac{8}{\eta})}\right].
\end{equation*} \\

    \item
     We have:
\begin{align*}
    &\mathbb{P}_{\mu^X_0}({h^*_{\eps',\delta',\text{Fair}}}_{0}(X) \neq h^*_{\epsilon,\delta,0}(X))+ \mathbb{P}_{\mu^X_1}({h^*_{\eps',\delta',\text{Fair}}}_{1}(X) \neq h^*_{\epsilon,\delta,1}(X))\\
    & \hspace{1cm} = \mathbb{P}_{\mu^X_0}\Bigl( (h^*_{\epsilon,\delta,0}(X)=1) \hspace{0.1cm} \text{and} \hspace{0.1cm} ({h^*_{\eps',\delta',\text{Fair}}}_{0}(X) = 0) \Bigr) +   \mathbb{P}_{\mu^X_1}\Bigl( (h^*_{\epsilon,\delta,1}(X)=0)  \hspace{0.1cm} \text{and} \hspace{0.1cm} ({h^*_{\eps',\delta',\text{Fair}}}_{1}(X) = 1) \Bigr) \\
&\hspace{1cm}=\alpha \left(  \frac{\tilde{\alpha}-\tilde{\beta}}{2\tilde{\alpha}} \right) + (1-\beta)\left(  \frac{\tilde{\alpha}-\tilde{\beta}}{2(1-\tilde{\beta})}\right) =  \frac{\alpha}{\tilde{\alpha}} \left(  \frac{\tilde{\alpha}-\tilde{\beta}}{2} \right) + \frac{(1-\beta)}{1-\tilde{\beta}}\left(  \frac{\tilde{\alpha}-\tilde{\beta}}{2}\right)  \\
&\hspace{1cm}= \frac{\tilde{\alpha}-e_0-d_0}{\tilde{\alpha}} \left(  \frac{\tilde{\alpha}-\tilde{\beta}}{2} \right) + \frac{(1-\tilde{\beta} + e_1 + d_1)}{1-\tilde{\beta}}\left(  \frac{\tilde{\alpha}-\tilde{\beta}}{2}\right)  \\
&\hspace{1cm}=(\tilde{\alpha} - \tilde{\beta}) + (e_1 + d_1)\left(\frac{\tilde{\alpha}-\tilde{\beta}}{2(1-\tilde{\beta})}\right) - (e_0+d_0)\left(  \frac{\tilde{\alpha}-\tilde{\beta}}{2\tilde{\alpha}} \right).
\end{align*}

Therefore, we have:
\begin{align*}
\begin{split}
    &\mathbb{P}_{\mu^X_0}({h^*_{\eps',\delta',\text{Fair}}}_{0}(X) \neq h^*_{\epsilon,\delta,0}(X))+  \mathbb{P}_{\mu^X_1}({h^*_{\eps',\delta',\text{Fair}}}_{1}(X) \neq h^*_{\epsilon,\delta,1}(X)) -  \text{err}^*(h^*_{\epsilon,\delta,0},h^*_{\epsilon,\delta,1})  \\
& \qquad \qquad = (\tilde{\alpha} - \tilde{\beta}) + (e_1 + d_1)\left(\frac{\tilde{\alpha}-\tilde{\beta}}{2(1-\tilde{\beta})}\right) -  (e_0+d_0)\left(  \frac{\tilde{\alpha}-\tilde{\beta}}{2\tilde{\alpha}} \right) - \text{err}^*(h^*_{\epsilon,\delta,0},h^*_{\epsilon,\delta,1}).
\end{split}
\end{align*}

From previous part, we know that:
\begin{equation*}
    \Delta_{SP}(h^*_{\eps',\delta',\text{Fair}}) \leq |e_0|+|e_1| + |d_0|+|d_1|. 
\end{equation*}
From Proposition \ref{proposition1}, we know that if $\Delta_{SP}(\hat{h}) \leq \gamma$, then:
\begin{equation*}
    \begin{split}
        & \mathbb{P}_{\mu^X_0}(\hat{h}_0(X) \neq h^*_{\epsilon,\delta,0}(X))+\mathbb{P}_{\mu^X_1}(\hat{h}_1(X) \neq h^*_{\epsilon,\delta,1}(X))  \geq \left| \mathbb{P}_{\mu^X_0}(h^*_{\epsilon,\delta,0}(X) = 1 ) - \mathbb{P}_{\mu^{X}_{1}}(h^*_{\epsilon,\delta,1}(X) = 1 )\right| - \gamma.
    \end{split}
\end{equation*}
For all classifiers $\hat{h}:\mathcal{X} \times \{0,1\} \rightarrow \{0,1\}$ that satisfy $\Delta_{SP}(\hat{h}) \leq \Delta_{SP}(h^*_{\eps',\delta',\text{Fair}})$, we have $\Delta_{SP}(\hat{h}) \leq |e_0|+|e_1| + |d_0|+|d_1|$. Thus, for all those classifiers we have:
\begin{equation*}
    \begin{split}
        & \mathbb{P}_{\mu^X_0}(\hat{h}_0(X) \neq h^*_{\epsilon,\delta,0}(X))+\mathbb{P}_{\mu^X_1}(\hat{h}_1(X) \neq h^*_{\epsilon,\delta,1})  \\
    & \qquad \qquad \qquad \geq \left| \mathbb{P}_{\mu^X_0}(h^*_{\epsilon,\delta,0}(X) = 1 ) - \mathbb{P}_{\mu^{X}_{1}}(h^*_{\epsilon,\delta,1}(X) = 1 )\right| -  \bigl(|e_0|+|e_1| + |d_0|+|d_1|\bigr) \\ &  \qquad \qquad \qquad  =\left|\alpha - \beta\right| - \bigr( |e_0|+|e_1| + |d_0|+|d_1| \bigl).
    \end{split}
\end{equation*}

Therefore, by definition of $\text{err}^*(h^*_{\epsilon,\delta,0},h^*_{\epsilon,\delta,1})$, we have:
\begin{equation*}\text{err}^*(h^*_{\epsilon,\delta,0},h^*_{\epsilon,\delta,1}) \geq |\alpha - \beta| - \left( |e_0|+|e_1| + |d_0|+|d_1| \right).
\end{equation*}
Therefore, it follows:
\begin{align*}
         &\mathbb{P}_{\mu^X_0}({h^*_{\eps',\delta',\text{Fair}}}_{0}(X) \neq h^*_{\epsilon,\delta,0}(X))+  
         \mathbb{P}_{\mu^X_1}({h^*_{\eps',\delta',\text{Fair}}}_{1}(X) \neq h^*_{\epsilon,\delta,1}(X)) - \text{err}^*(h^*_{\epsilon,\delta,0},h^*_{\epsilon,\delta,1})  \\
    & \qquad \quad \leq (\tilde{\alpha} - \tilde{\beta}) + (e_1 + d_1)\left(\frac{\tilde{\alpha}-\tilde{\beta}}{2(1-\tilde{\beta})}\right) -  (e_0+d_0)\left(  \frac{\tilde{\alpha}-\tilde{\beta}}{2\tilde{\alpha}} \right)   - |\alpha - \beta|+ \left( |e_0|+|e_1| + |d_0|+|d_1| \right) \\
    & \qquad \quad\leq \frac{1}{2}|e_1+d_1| + \frac{1}{2}|e_0+d_0| + (\tilde{\alpha} - \tilde{\beta}) - |\alpha-\beta|+ \left( |e_0|+|e_1| + |d_0|+|d_1| \right)  \\
    &\qquad \quad\leq (\tilde{\alpha} - \tilde{\beta}) - |\alpha-\beta|+ \frac{3}{2}\left( |e_0|+|e_1| + |d_0|+|d_1| \right)  \\
    & \qquad \quad = (\tilde{\alpha} - \tilde{\beta}) - |\tilde{\alpha} - e_0 - d_0 - \tilde{\beta} + e_1 + d_1|+ \frac{3}{2}\left( |e_0|+|e_1| + |d_0|+|d_1| \right)  \\ 
    & \qquad \quad \leq  \frac{5}{2}\left( |e_0|+|e_1| + |d_0|+|d_1| \right).
\end{align*}

By the same argument of the second part, we can conclude that:

\begin{align*}
    & \mathbb{P}_{\mu^X_0} ({h^*_{\eps',\delta',\text{Fair}}}_{0}(X) \neq h^*_{\epsilon,\delta,0}(X))    +  \mathbb{P}_{\mu^X_1}({h^*_{\eps',\delta',\text{Fair}}}_{1}(X) \neq h^*_{\epsilon,\delta,1}(X))   \nonumber \\ 
    & \quad \leq_{\eta} \text{err}^*(h^*_{\epsilon,\delta,0},h^*_{\epsilon,\delta,1}) + \frac{5}{2}\left(     \left(\log{\frac{4}{\eta}}\right)\left(\frac{1}{n \theta\eps_0}\right) + \left(\log{\frac{4}{\eta}}\right)\left(\frac{1}{n \bar{\theta}\eps_1}\right) + \sqrt{\frac{1}{2n \theta}\log(\frac{8}{\eta})} + \sqrt{\frac{1}{2n \bar{\theta}}\log(\frac{8}{\eta})} \right).
\end{align*}
\end{itemize}  
\end{proof}

\begin{proof}[Proof of Proposition \ref{alg2Exp}]
 Let $\alpha = \mathbb{P}_{\mu^X_0} (h^*_{\epsilon,\delta,0}(X) = 1 )$ and $\beta = \mathbb{P}_{\mu^X_1}(h^*_{\epsilon,\delta,1}(X) = 1 )$. Also, let $\bar{\alpha} = \alpha + e_0$ and $\bar{\beta} = \beta + e_1$. 
     Let $L_0 \sim \text{Lap}\big(\frac{1}{n \theta \eps_0}\big)$ and $L_1 \sim \text{Lap}\big(\frac{1}{n \bar{\theta} \eps_1}\big)$. We sample $l_0$ from $L_0$ and $l_1$ from $L_1$.  In fact, $l_0$ and $l_1$ are realizations of the Laplace noise. From proof of Theorem \ref{alg2WHP}, we know that for each realization of the noise we have: 
\begin{equation*}
    \begin{split}
        \Delta_{SP}(h^*_{\eps',\delta',\text{Fair}})  \leq  \left[ |e_0|+|e_1| + |l_0|+|l_1| \right].
    \end{split}
\end{equation*}
And
\begin{align*}
         & \Biggl[  \mathbb{P}_{\mu^X_0}({h^*_{\eps',\delta',\text{Fair}}}_{0}(X) \neq h^*_{\epsilon,\delta,0}(X))+  \mathbb{P}_{\mu^X_1}({h^*_{\eps',\delta',\text{Fair}}}_{1}(X)  \neq  h^*_{\epsilon,\delta,1}(X)) - \text{err}^*(h^*_{\epsilon,\delta,0},h^*_{\epsilon,\delta,1}) \Biggr]  \leq  \\
         & \hspace{11.2cm} \frac{5}{2}  \left[ \left( |e_0|+|e_1| + |l_0|+|l_1| \right) \right].
\end{align*}
We have:
\begin{equation}\label{eq: expected delta_SP}
    \mathbb{E}\left[\Delta_{SP}(h^*_{\eps',\delta',\text{Fair}})\right]  \leq  \mathbb{E}\left[ |e_0|+|e_1| + |L_0|+|L_1| \right].
\end{equation}
And
\begin{equation}\label{eq:expected utility}
    \begin{split}
        \mathbb{E}\Biggl[  \mathbb{P}_{\mu^X_0}({h^*_{\eps',\delta',\text{Fair}}}_{0}(X) \neq h^*_{\epsilon,\delta,0}(X)) +   \mathbb{P}_{\mu^X_1}({h^*_{\eps',\delta',\text{Fair}}}_{1}(X) \neq &  h^*_{\epsilon,\delta,1}(X)) - \text{err}^*(h^*_{\epsilon,\delta,0},h^*_{\epsilon,\delta,1}) \Biggr] \leq \\
        & \qquad \qquad \qquad \qquad\frac{5}{2}  \mathbb{E}\left[ \left( |e_0|+|e_1| + |L_0|+|L_1| \right) \right].
    \end{split}
\end{equation}

Where the expectations are over the randomness of the Laplace noise and randomness of the approximation of $\alpha$ and $\beta$ with finite samples. For $L \sim \text{Lap}(\frac{\Delta_1^q}{ \epsilon})$, we have $\mathbb{E}\left[ |L|\right] = \frac{\Delta_1^q}{ \epsilon}$. Since $L_0 \sim \text{Lap}\left(\frac{1}{n \theta \eps_0}\right) \text{ and } L_1 \sim \text{Lap}\left(\frac{1}{n \bar{\theta} \eps_1}\right)$, we have:
\begin{equation}\label{eq: Expected Laplace}
    \mathbb{E}\left[|L_0| + |L_1|  \right] = \left(\frac{1}{n \theta \eps_0}\right) + \left(\frac{1}{n \bar{\theta} \eps_1}\right).
\end{equation}

Now, we want to find an upper bound for $\mathbb{E}[|e_0|]$ and $\mathbb{E}[|e_1|]$. We have:
\begin{align*}
    \mathbb{E}[|e_0|] ^ 2 & = \mathbb{E}\left[ \left| \frac{1}{n \theta}\sum\limits_{\substack{i=1\\ A_i = 0}}^{n} h^*_{\epsilon,\delta,0}(X_i) -  \mathbb{P}_{\mu^X_0} (h^*_{\epsilon,\delta,0}(X) = 1 )\right| \right]^2  \\  & \leq  \mathbb{E}\left[ \left| \frac{1}{n \theta}\sum\limits_{\substack{i=1\\ A_i = 0}}^{n} h^*_{\epsilon,\delta,0}(X_i) -  \mathbb{P}_{\mu^X_0} (h^*_{\epsilon,\delta,0}(X) = 1 )\right| ^ 2 \right] \\
    & = \mathbb{E}\left[ \left( \frac{1}{n \theta}\sum\limits_{\substack{i=1\\ A_i = 0}}^{n} h^*_{\epsilon,\delta,0}(X_i) -  \mathbb{E}\left[ \frac{1}{n \theta}\sum\limits_{\substack{i=1\\ A_i = 0}}^{n} h^*_{\epsilon,\delta,0}(X_i)\right]\right) ^ 2 \right] \\
    & = \mathrm{Var}\left(\frac{1}{n \theta}\sum\limits_{\substack{i=1\\ A_i = 0}}^{n} h^*_{\epsilon,\delta,0}(X_i)\right) \\
    & = \frac{1}{\theta^2 n^2}\mathrm{Var}\left(\sum\limits_{\substack{i=1\\ A_i = 0}}^{n} h^*_{\epsilon,\delta,0}(X_i)\right) \\
    & = \frac{1}{n \theta}\mathrm{Var}\left( h^*_{\epsilon,\delta,0}(X_i)\right) \\
    & \leq \frac{1}{4 n \theta}.
\end{align*}

The inequality in the second line is due to the property that for any random variable $Z$, we have $\mathbb{E}[|Z|]^2 \leq \mathbb{E}[Z^2]$, and the inequality in the last line follows from $\mathrm{Var}(\text{Bernoulli}(p)) = p(1-p) \leq \frac{1}{4}$ for $0 \leq p \leq 1$.

Similarly, it follows that:
\begin{equation*}
    \mathbb{E} [|e_1|]^2 \leq \frac{1}{4 n \bar{\theta}}.
\end{equation*}

Therefore, it can be shown that:
\begin{equation}\label{eq: expected approx error}
    \mathbb{E} [|e_0|]  + \mathbb{E} [|e_1|]\leq \sqrt{\frac{1}{4 n \theta}} + \sqrt{\frac{1}{4 n \bar{\theta}}}.
\end{equation}

Combining (\ref{eq: expected delta_SP}), (\ref{eq: Expected Laplace}), and (\ref{eq: expected approx error}), we have:

\begin{equation*}
    \mathbb{E}\left[\Delta_{SP}(h^*_{\eps',\delta',\text{Fair}})\right]  \leq \frac{1}{n \theta \eps_0} + \frac{1}{n \bar{\theta} \eps_1} + \sqrt{\frac{1}{4 n \theta}} + \sqrt{\frac{1}{4 n \bar{\theta}}}.
\end{equation*}
With the same argument and Using (\ref{eq:expected utility}), it can be shown that:
\begin{align*}
        \mathbb{E}\Biggl[ \mathbb{P}_{\mu^X_0}({h^*_{\eps',\delta',\text{Fair}}}_{0}(X) \neq h^*_{\epsilon,\delta,0}(X)) & +  \mathbb{P}_{\mu^X_1}({h^*_{\eps',\delta',\text{Fair}}}_{1}(X) \neq h^*_{\epsilon,\delta,1}(X))  \Biggr] \leq \\ & \qquad \mathbb{E}\Bigl[\text{err}^*(h^*_{\epsilon,\delta,0},h^*_{\epsilon,\delta,1})\Bigr] +
         \frac{5}{2}\left(\frac{1}{n \theta \eps_0} + \frac{1}{n \bar{\theta} \eps_1} + \sqrt{\frac{1}{4 n \theta}} + \sqrt{\frac{1}{4 n \bar{\theta}}}\right).
\end{align*}
\end{proof}

\section*{Appendix B}\label{appendix_experiments}

In this section, we provide additional details on our experiments, including datasets, pre-processing approach, and additional experimental results.

\textbf{Adult Dataset:} This dataset contains census information about individuals. The prediction task is to determine whether a person earns over \$50K a year. In this dataset, gender is considered a sensitive attribute. For all experiments, we followed a pre-processing approach similar to \cite{lowy2023stochastic}. After pre-processing, the dataset contains a total of 102 input features. The size of the dataset is around 48,000 entries.

\textbf{Credit Card Dataset:}
This dataset includes financial data from bank users in Taiwan. The prediction task is to assess whether a person defaults on their credit card bills, essentially evaluating client credibility. Gender is taken as a sensitive attribute. We applied the same pre-processing method as in \cite{lowy2023stochastic} for all experiments. After pre-processing, the dataset contained 85 input features, with a total of 30,000 entries.

For DP-FERMI experiments, datasets were split into $75\%$ training and $25\%$ testing. However, for Algorithm \ref{algorithm2}, we needed to train the initial classifier $h^*_{\epsilon,\delta}:\mathcal{X} \times \{0,1\} \to \{0,1\}$  followed by a post-processing step. For this, we divided the dataset into three parts: $50\%$ for training the initial classifier, and $25\%$ each for post-processing and testing.

To calculate the statistical parity gap of output classifiers in all experiments, we refer to Definition \ref{statistical_parity_gap}. We need to compute:
\begin{equation*}
    \Delta_{SP}(\hat{h}) := |\mu_0(\hat{Y} = 1) - \mu_1(\hat{Y} = 1)|.
\end{equation*}
In practice, where we do not have access to true distributions $\mu_a$, we use their empirical counterparts. Specifically, we calculate the empirical version of $\mu_a(\hat{Y} = 1)$, i.e., we compute $\hat{\mathbb{P}}\bigl[ \hat{Y} = 1 | A = a \bigr]$. Thus, the empirical version of the statistical parity gap is defined as:
\begin{equation*}
    \widehat{\Delta_{SP}}(\hat{h}) := \Biggl|\hat{\mathbb{P}}\Bigl[ \hat{Y} = 1 | A = 0 \Bigr] - \hat{\mathbb{P}}\Bigl[ \hat{Y} = 1 | A = 1 \Bigr]\Biggr|.
\end{equation*}
In this expression, $\hat{\mathbb{P}}\bigl[ \hat{Y} = 1 | A = 0 \bigr]$ and $\hat{\mathbb{P}}\bigl[ \hat{Y} = 1 | A = 1 \bigr]$ are computed on the test split of the dataset. The statistical parity gap values listed in the column of all tables are derived using the empirical statistical parity gap calculation.


We present the results of our experiments with privacy parameters computed using both the moments accountant and GDP accountant methods. For moments accountant method, we selected the standard deviations of Gaussian noise in DP-SGD to be $3.27$ and $1.56$ for the Adult dataset, and $4.65$ and $2.6$ for the Credit Card dataset. These values were chosen to achieve their respective privacy parameters of $(\eps' = 3, \delta' = 10^{-5})$ and $(\eps' = 9, \delta' = 10^{-5})$. To achieve the target privacy parameters using GDP accountant in Algorithm 2, we set specific noise variances for the Adult and Credit Card datasets. For the Adult dataset, with privacy parameters $(\eps'=3,\delta'=10^{-5})$ and $(\eps'=9,\delta'=10^{-5})$, the standard deviation of the Gaussian noise were set to $3.01$ and $1.45$, respectively. In the case of the Credit Card dataset, under identical privacy parameters, we used a Gaussian noise with a standard deviation of $4.3$ and $2.43$ in DP-SGD. All other parameters, including $\eps_0$, $\eps_1$, learning rate, number of epochs, batch size, and clipping constant, were kept consistent with those used in the experiments where the privacy guarantee was calculated using PRV accountant method (the experiments detailed in Tables \ref{tab:Adult_Eps_3_PRV}, \ref{tab:Adult_Eps_9_PRV}, \ref{tab:Credit_Eps_3_PRV}, and \ref{tab:Credit_Eps_9_PRV}).

The results corresponding to moments accountant method are presented in Tables \ref{tab:Adult_Eps_3}, \ref{tab:Adult_Eps_9}, \ref{tab:Credit_Eps_3}, and \ref{tab:Credit_Eps_9}. Additionally, the results corresponding to GDP accountant method are available in Tables \ref{tab:Adult_Eps_3_GDP}, \ref{tab:Adult_Eps_9_GDP}, \ref{tab:Credit_Eps_3_GDP}, and \ref{tab:Credit_Eps_9_GDP}. Note that the DP-FERMI experiment results remain the same; the only difference lies in the line corresponding to Algorithm \ref{algorithm2}.

It can be concluded that across all these experiments, similar to what was noted in Section \ref{experiments}, Algorithm \ref{algorithm2} consistently offers a substantially better (smaller) statistical parity gap compared to DP-FERMI for fixed privacy and accuracy, regardless of varying privacy parameters, datasets (Adult and Credit Card), and privacy accounting methods employed.

\begin{table}[t]
\centering
\renewcommand{\arraystretch}{1.2} 
\begin{minipage}{0.5\textwidth}
  \centering
  \caption{Adult Dataset ($\eps' = 3, \delta' = 10^{-5}$) - Moments Accountant}
  \label{tab:Adult_Eps_3}
  \begin{tabular}{|
  >{\centering\arraybackslash}m{0.4\columnwidth}|
  >{\centering\arraybackslash}m{0.2\columnwidth}|
  >{\centering\arraybackslash}m{0.2\columnwidth}|}
    \hline
    Method & Accuracy & Statistical Parity Gap \\
    \hline
    Algorithm 2 & 0.7752 & 0.0057\\
    \hline
    DP-FERMI ($\lambda=0.5$) & 0.7998 & 0.1020\\
    \hline
    DP-FERMI ($\lambda=1$) & 0.7859 & 0.0462\\
    \hline
    DP-FERMI ($\lambda=1.5$) & 0.7822 & 0.0267\\
    \hline
    DP-FERMI ($\lambda=1.9$) &  0.7749 & 0.0126\\
    \hline
    DP-FERMI ($\lambda=2.5$) & 0.7673 & 0.0099\\
    \hline
  \end{tabular}
\end{minipage}%
\begin{minipage}{0.5\textwidth}
  \centering
  \caption{Adult Dataset ($\eps' = 9, \delta' = 10^{-5}$) - Moments Accountant}
  \label{tab:Adult_Eps_9}
  \begin{tabular}{|
  >{\centering\arraybackslash}m{0.4\columnwidth}|
  >{\centering\arraybackslash}m{0.2\columnwidth}|
  >{\centering\arraybackslash}m{0.2\columnwidth}|}
    \hline
    Method & Accuracy & Statistical Parity Gap \\
    \hline
    Algorithm 2 & 0.7782 & 0.0054\\
    \hline
    DP-FERMI ($\lambda=0.5$) & 0.8091 & 0.0944\\
    \hline
    DP-FERMI ($\lambda=1$) & 0.7923 & 0.0413\\
    \hline
    DP-FERMI ($\lambda=1.5$) & 0.7810 & 0.0152\\
    \hline
    DP-FERMI ($\lambda=1.7$) & 0.7782 & 0.0121\\
    \hline
    DP-FERMI ($\lambda=2.5$) & 0.7693 & 0.0030\\
    \hline
  \end{tabular}
\end{minipage}
\end{table}

\begin{table}[t]
\centering
\renewcommand{\arraystretch}{1.2} 
\begin{minipage}{0.5\textwidth}
  \centering
  \caption{Credit Card Dataset ($\eps' = 3, \delta' = 10^{-5}$) - Moments Accountant}
  \label{tab:Credit_Eps_3}
  \begin{tabular}{|
  >{\centering\arraybackslash}m{0.4\columnwidth}|
  >{\centering\arraybackslash}m{0.2\columnwidth}|
  >{\centering\arraybackslash}m{0.2\columnwidth}|}
    \hline
    Method & Accuracy & Statistical Parity Gap \\
    \hline
    Algorithm 2 & 0.7842 & 0.0041\\
    \hline
    DP-FERMI ($\lambda=0.1$) & 0.7899 & 0.0212\\
    \hline
    DP-FERMI ($\lambda=0.2$) & 0.7846 & 0.0193\\
    \hline
    DP-FERMI ($\lambda=0.5$) & 0.7777 & 0.0185\\
    \hline
    DP-FERMI ($\lambda=1$) & 0.7759 & 0.0105\\
    \hline
    DP-FERMI ($\lambda=2.5$) & 0.7669 & 0.0110\\
    \hline
  \end{tabular}
\end{minipage}%
\begin{minipage}{0.5\textwidth}
  \centering
  \caption{Credit Card Dataset ($\eps' = 9, \delta' = 10^{-5}$) - Moments Accountant}
  \label{tab:Credit_Eps_9}
  \begin{tabular}{|
  >{\centering\arraybackslash}m{0.4\columnwidth}|
  >{\centering\arraybackslash}m{0.2\columnwidth}|
  >{\centering\arraybackslash}m{0.2\columnwidth}|}
    \hline
    Method & Accuracy & Statistical Parity Gap \\
    \hline
    Algorithm 2 & 0.7908 & 0.0071\\
    \hline
    DP-FERMI ($\lambda=0.25$) & 0.7996 & 0.0188\\
    \hline
    DP-FERMI ($\lambda=0.35$) & 0.7950 & 0.0174\\
    \hline
    DP-FERMI ($\lambda=0.5$) & 0.7912 & 0.0172\\
    \hline
    DP-FERMI ($\lambda=1$) & 0.7895 & 0.0105\\
    \hline
    DP-FERMI ($\lambda=2.5$) & 0.7884 & 0.0066\\
    \hline
  \end{tabular}
\end{minipage}
\end{table}

\begin{table}[H]
\centering
\renewcommand{\arraystretch}{1.2} 
\begin{minipage}{0.5\textwidth}
  \centering
  \caption{Adult Dataset ($\eps' = 3, \delta' = 10^{-5})$ - GDP Accountant}
  \label{tab:Adult_Eps_3_GDP}
  \begin{tabular}{|
  >{\centering\arraybackslash}m{0.4\columnwidth}|
  >{\centering\arraybackslash}m{0.2\columnwidth}|
  >{\centering\arraybackslash}m{0.2\columnwidth}|}
    \hline
    Method & Accuracy & Statistical Parity Gap \\
    \hline
    Algorithm 2 & 0.7796 & 0.0081\\
    \hline
    DP-FERMI ($\lambda=0.5$) & 0.7998 & 0.1020\\
    \hline
    DP-FERMI ($\lambda=1$) & 0.7859 & 0.0462\\
    \hline
    DP-FERMI ($\lambda=1.5$) & 0.7822 & 0.0267\\
    \hline
    DP-FERMI ($\lambda=1.7$) &  0.7795 & 0.0215\\
    \hline
    DP-FERMI ($\lambda=2.5$) & 0.7673 & 0.0099\\
    \hline
  \end{tabular}
\end{minipage}%
\begin{minipage}{0.5\textwidth}
  \centering
  \caption{Adult Dataset ($\eps' = 9, \delta' = 10^{-5})$ - GDP Accountant}
  \label{tab:Adult_Eps_9_GDP}
  \begin{tabular}{|
  >{\centering\arraybackslash}m{0.4\columnwidth}|
  >{\centering\arraybackslash}m{0.2\columnwidth}|
  >{\centering\arraybackslash}m{0.2\columnwidth}|}
    \hline
    Method & Accuracy & Statistical Parity Gap \\
    \hline
    Algorithm 2 & 0.7804& 0.0060\\
    \hline
    DP-FERMI ($\lambda=0.5$) & 0.8091 & 0.0944\\
    \hline
    DP-FERMI ($\lambda=1$) & 0.7923 & 0.0413\\
    \hline
    DP-FERMI ($\lambda=1.5$) & 0.7810 & 0.0152\\
    \hline
    DP-FERMI ($\lambda=1.8$) & 0.7769 & 0.0113\\
    \hline
    DP-FERMI ($\lambda=2.5$) & 0.7693 & 0.0030\\
    \hline
  \end{tabular}
\end{minipage}
\end{table}

\begin{table}[H]
\centering
\renewcommand{\arraystretch}{1.2} 
\begin{minipage}{0.5\textwidth}
  \centering
  \caption{Credit Card Dataset ($\eps' = 3, \delta' = 10^{-5})$ -  GDP Accountant}
  \label{tab:Credit_Eps_3_GDP}
  \begin{tabular}{|
  >{\centering\arraybackslash}m{0.4\columnwidth}|
  >{\centering\arraybackslash}m{0.2\columnwidth}|
  >{\centering\arraybackslash}m{0.2\columnwidth}|}
    \hline
    Method & Accuracy & Statistical Parity Gap \\
    \hline
    Algorithm 2 & 0.7863 & 0.0062\\
    \hline
    DP-FERMI ($\lambda=0.1$)  & 0.7899 & 0.0212\\
    \hline
    DP-FERMI ($\lambda=0.17$)  & 0.7860  & 0.0199\\
    \hline
    DP-FERMI ($\lambda=0.5$)  & 0.7777 & 0.0185\\
    \hline
    DP-FERMI ($\lambda=1$) & 0.7759  & 0.0105\\
    \hline
    DP-FERMI ($\lambda=2.5$) & 0.7669 & 0.0110 \\
    \hline
  \end{tabular}
\end{minipage}%
\begin{minipage}{0.5\textwidth}
  \centering
  \caption{Credit Card Dataset ($\eps' = 9, \delta' = 10^{-5})$ - GDP Accountant}
  \label{tab:Credit_Eps_9_GDP}
  \begin{tabular}{|
  >{\centering\arraybackslash}m{0.4\columnwidth}|
  >{\centering\arraybackslash}m{0.2\columnwidth}|
  >{\centering\arraybackslash}m{0.2\columnwidth}|}
    \hline
    Method & Accuracy & Statistical Parity Gap \\
    \hline
    Algorithm 2 & 0.7915 & 0.0063\\
    \hline
    DP-FERMI ($\lambda=0.25$)  & 0.7996 & 0.0188 \\
    \hline
    DP-FERMI ($\lambda=0.35$) & 0.7950 & 0.0174\\
    \hline
    DP-FERMI ($\lambda=0.5$) & 0.7912 & 0.0172 \\
    \hline
    DP-FERMI ($\lambda=1$) & 0.7895 & 0.0105\\
    \hline
    DP-FERMI ($\lambda=2.5$) & 0.7884 & 0.0066\\
    \hline
  \end{tabular}
\end{minipage}
\end{table}

\end{document}